\documentclass[a4paper]{article}

\usepackage{natbib}
\usepackage{pdfsync, paralist}
\usepackage{enumitem}
\usepackage{amsmath}
\usepackage{amssymb}
\usepackage{amsthm}
\usepackage{color}
\usepackage{hyperref}
\usepackage[titletoc,toc,title]{appendix}
\usepackage{bbm}
\usepackage{bm}
\usepackage{mathtools}
\usepackage{algorithmic}
\usepackage[ruled,vlined]{algorithm2e}
\usepackage{fullpage}

\DeclareMathOperator*{\argmax}{arg\,max} 

\setlist{leftmargin=*}

\hypersetup{
    colorlinks=true,
    linkcolor=blue,
    filecolor=blue,      
    urlcolor=blue,
    citecolor= blue
}

\newtheorem{theorem}{Theorem}[section]
\newtheorem{lemma}[theorem]{Lemma}

\newtheorem{corollary}[theorem]{Corollary}
\newtheorem{myconjecture}{Conjecture}
\newtheorem{definition}[theorem]{Definition}

\newcommand\ee{\mathbf{e}}
\newcommand{\ind}{\mathbf{1}}

\newcommand{\cD}{\mathcal{D}}

\newcommand{\PP}{\mathbb{P}}
\newcommand{\EE}{\mathbb{E}}
\newcommand{\RR}{\mathbb{R}}

\newcommand{\Regret}{\mathrm{Regret}}

\newcommand{\XX}{G}
\newcommand{\xx}{g}
\newcommand\lh{\hat{\xx}}
\newcommand\Lh{\hat{\XX}}

\newcommand\Lt{\tilde{\XX}}

\newcommand{\f}{\Phi}
\newcommand{\tf}{{\tilde{\f}}}
\newcommand{\gtf}{\nabla {\tilde{\f}}}

\newcommand{\D}{\Delta}

\newcommand\inner[1]{\left\langle #1 \right\rangle}

\newenvironment{framework}[1][htb]
  {
   \begin{algorithm}[#1]%
  }{\end{algorithm}}

\title{\textbf{Beyond the Hazard Rate: More Perturbation Algorithms for Adversarial Multi-armed Bandits}}

\author{
Zifan Li \\
University of Michigan\\
\texttt{zifanli@umich.edu} \\
\and
Ambuj Tewari \\
University of Michigan\\
\texttt{tewaria@umich.edu} \\
}

\begin{document}
\maketitle


\begin{abstract}
Recent work on follow the perturbed leader (FTPL) algorithms for the adversarial multi-armed bandit problem
has highlighted the role of the hazard rate of the distribution generating the perturbations. Assuming that
the hazard rate is bounded, it is possible to provide regret analyses for a variety of FTPL algorithms for the
multi-armed bandit problem. This paper pushes the inquiry into regret bounds for FTPL algorithms beyond the
bounded hazard rate condition. There are good reasons to do so: natural distributions such as the uniform and Gaussian
violate the condition. We give regret bounds for both bounded support and unbounded support distributions without
assuming the hazard rate condition. We also disprove a conjecture that the Gaussian distribution cannot lead to a low-regret
algorithm. In fact, it turns out that it leads to near optimal regret, up to logarithmic factors. A key ingredient in our approach
is the introduction of a new notion called the generalized hazard rate. \\
\end{abstract}
Keywords: online learning, regret, multi-armed bandits, follow the perturbed leader, gradient based algorithms\\


\section{Introduction}

Starting from the seminal work of \cite{Hannan57} and later developments due to \cite{KV-FTL}, perturbation based algorithms (called ``Follow the Perturbed Leader (FTPL)") have occupied a central place in online learning.
Another major family of online learning algorithms, called ``Follow the Regularized Leader (FTRL)", is based on the idea of regularization. In special cases, such as the exponential
weights algorithm for the experts problem, it has been folk knowledge that regularization and perturbation ideas are connected.
That is, the exponential weights algorithm can be understood as either using negative entropy regularization or Gumbel distributed perturbations (for example, see the discussion in \cite{abernethy2014online}).
 
Recent work have begun to further uncover the connections between perturbation and regularization. For example, in online linear optimization,
one can understand regularization and perturbation as simply two different ways to smooth a non-smooth potential function.
The former corresponds to infimal convolution smoothing and the latter corresponds to stochastic (or integral convolution) smoothing \citep{abernethy2014online}. Having a generic framework for understanding
perturbations allows one to study a wide variety of online linear optimization games and a number of interesting perturbations.

FTRL and FTPL algorithms have also been used beyond ``full information" settings. ``Full information" refers to the fact that the
learner observes the entire move of the adversary. The multi-armed bandit problem is one of the most fundamental examples of ``partial information" settings.
Regret analysis of the multi-armed bandit problem goes back
to the work of \cite{Robbins52} who formulated the stochastic version of the problem. The non-stochastic, or adversarial, version was formulated by \cite{Auer2002},
who provided the EXP3 algorithm achieving $O(\sqrt{N T \log N})$ regret in $T$ rounds with $N$ arms. They also showed a lower bound of $\Omega(\sqrt{NT})$, 
which was later matched by the Poly-INF algorithm \citep{audibert2009minimax,audibert2011minimax}. The Poly-INF algorithm can be interpreted as an FTRL
algorithm with negative Tsallis entropy regularization \citep{audibert2011minimax,Abernethy2015}.
For a recent survey of both stochastic and non-stochastic bandit problems, see \cite{bubeck2012regret}.

For the non-stochastic multi-armed bandit problem, \cite{kujala2005following} and \cite{Poland2005} both showed that using
the exponential (actually double exponential/Laplace) distribution
in an FTPL algorithm coupled with standard unbiased estimation technique yields near-optimal $O(\sqrt{NT\log N})$ regret.
Unbiased estimation needs access to arm probabilities that are not explicitly available when using an FTPL algorithm.
\cite{neu2013efficient} introduced the geometric resampling scheme to approximate these probabilities while still guaranteeing low regret.
Recently, \cite{Abernethy2015} analyzed FTPL for adversarial multi-armed
bandits and provided regret bounds under the condition that the hazard rate of the perturbation distribution is bounded. This condition
allowed them to consider a variety of perturbation distributions beyond the exponential, such as Gamma, Gumbel, Frechet, Pareto, and Weibull.

Unfortunately, the bounded hazard rate condition is violated by two of the most widely known distributions: namely the uniform\footnote{The uniform distribution
is also historically significant as it was used in the original FTPL algorithm of \cite{Hannan57}.}  and the Gaussian
distributions. Therefore, the results of \cite{Abernethy2015} say nothing about the regret incurred in an adversarial multi-armed bandit problem when we use 
these distributions (without forced exploration) to generate perturbations. Contrast this to the full information experts setting where using these distributions as perturbations yields
optimal $\sqrt{T}$ regret and even yields the optimal $\sqrt{\log N}$ dependence on the dimension in the Gaussian case \citep{abernethy2014online}.

The Gaussian distribution has lighter tails than the exponential. The hazard rate of a Gaussian increases linearly on the real line (and is hence unbounded)
whereas the exponential has a constant hazard rate. Does having too light a tail make a perturbation inherently bad? The uniform is even worse from a light tail point of view: it has
bounded support! In fact, \cite{kujala2005following} had trouble dealing with the uniform distribution and remarked, ``we failed to analyze
the expert setting when the perturbation distribution was uniform."
Does having a bounded support make a perturbation even worse? Or is it that the hazard rate condition is just a sufficient condition without being anywhere close
to necessary for a good regret bound to exist. The analysis of \cite{Abernethy2015} suggests that perhaps a bounded hazard rate is critical. They even made the following conjecture.

\begin{myconjecture}
\label{conj:hazard}
If a distribution $\cD$ has a monotonically increasing hazard rate $h_{\cD}(x)$ that does
not converge as $x \to +\infty$ (e.g., Gaussian), then there is a sequence of gains that causes the corresponding FTPL algorithm to incur at least a linear regret.
\end{myconjecture}

The main contribution of this paper is to provide answers to the questions raised above.
First, we show that boundedness of the hazard rate is certainly not a requirement for achieving sublinear (in $T$) regret.
Bounded support distributions, like the uniform, violate the boundedness condition on the hazard rate in the most extreme way.
Their hazard rate blows up not just asymptotically at infinity, as in the Gaussian case, but as one approaches the right edge of the support.
Yet, we can show (Corollary~\ref{cor:uniform}) that using the uniform distribution results in a regret bound of $O((NT)^{2/3})$. This bound is clearly not optimal.
But optimality is not the point here. What is surprising, especially if one regards Conjecture~\ref{conj:hazard} as plausible, is that a non-trivial sublinear bound holds
at all. In fact, we show (Corollary~\ref{cor:bounded}) that using \emph{any} continuous distribution with bounded support and bounded density results in a sublinear
regret bound.

Second, moving beyond bounded support distributions to ones with unbounded support, we settle Conjecture~\ref{conj:hazard} in the negative. In Theorem~\ref{thm:gaussian}
we show that, instead of suffering linear regret as predicted by Conjecture~\ref{conj:hazard}, a perturbation algorithm using the Gaussian distribution enjoys a near optimal
regret bound of $O(\sqrt{N T \log N} \log T)$. A key ingredient in our approach is a new quantity that we call the \emph{generalized
hazard rate} of a distribution. We show that bounded generalized hazard rate is enough to guarantee sublinear regret in $T$ (Theorem~\ref{thm:master}).

Finally, we investigate the relationship between tail behavior of random perturbations and the regret they induce. We show that heavy tails, along with some fairly mild assumptions, guarantee a bounded hazard rate (Theorem~\ref{thm:gen_hazard_bound_heavy})
and hence previous results can yield regret bounds for these perturbations. However, light tails can fail to have a bounded hazard rate. Nevertheless, we show that under
reasonable conditions, light tailed distributions do have a bounded \emph{generalized} hazard rate (Theorem~\ref{thm:gen_hazard_bound_light}). This result allows us to show that reasonably behaved light-tailed distributions
lead to near optimal regret (Corollary~\ref{cor:lighttailed}). In particular, the exponential power (or generalized normal) family of distributions yields near optimal regret (Theorem~\ref{thm:exp_power})

\section{Follow the Perturbed Leader Algorithm for Bandits}

Recall the setting of the adversarial multi-armed bandit problem \citep{Auer2002}. An adversary (or Nature) chooses gain vectors $\xx_t \in [-1,0]^N$ for $1 \le t \le T$ ahead of the game. Such an adversary is called {\em oblivious}.
At round $t = 1,\ldots, T$ in a repeated game, the learner must choose a distribution $p_t \in \D_N$ over the set of $N$ available arms (or actions).  The learner plays action $i_t$ sampled according to $p_t$ and accumulates the gains $\xx_{t,{i_t}} \in [-1,0]$. The learner observes only $\xx_{t,{i_t}}$ and receives no information about the values $\xx_{t,j}$ for $j \ne i_t$. 

The learner's goal is to minimize the \emph{regret}. Regret is defined to be the difference in the realized gains and the gains of the best fixed action in hindsight:
\begin{equation} \label{eq:regret}
	\Regret_T := \max_{i \in [N]} \sum_{t=1}^T (\xx_{t,i} - \xx_{t,i_t}).
\end{equation}
To be precise, we consider the \emph{expected} regret, where the expectation is taken with respect to the learner's randomization. Note that, under an oblivious adversary, the only random variables in the above expression are the actions $i_t$ of the learner. For convenience, define the cumulative gain vectors $\XX_t , t=1,2,\dots,T$ by
$$
\XX_t := \sum_{s=1}^t \xx_s.
$$

\subsection{The Gradient-Based Algorithmic Template}

We will consider the algorithmic template described in Framework~\ref{alg:gbpa_bandit}, which is the Gradient Based Prediction Algorithm (GBPA) (see, for example, \cite{Abernethy2015}). Let $\Delta^N$ be the $(N-1)$-dimensional probability simplex 
in $\RR^N$. Denote the standard basis vector along the $i$th dimension by $\ee_i$.
At any round $t$, the action choice $i_t$ is made by sampling from the distribution $p_t$ which is obtained by applying the gradient of a convex function $\tf$ to the estimate $\Lh_{t-1}$ of the cumulative gain vector so far.
The choice of $\tf$ is flexible but it must be a differentiable convex function such that its gradient is always in $\Delta^N$.

Note that we do not require the range of $\gtf$ be contained in the \emph{interior}
of the probability simplex. If we required the gradient to lie in the interior, we would not be able to deal with bounded support distributions such as the uniform distribution.
Even though some entries of the probability vector $p_t$ might be $0$, the estimation step is always well defined since $p_{t,i_t} > 0$.
But allowing $p_{t,i}$ to be zero means that $\lh_t$ is not exactly an unbiased estimator of $\xx_t$. Instead, it is an unbiased estimator on the support of $p_{t}$. That is,
$\EE[ \lh_{t,i} | i_{1:t-1} ] = \xx_{t,i}$ for any $i$ such that $p_{t,i} > 0$. Here, $i_{1:t-1}$ is shorthand for $i_1,\ldots,i_{t-1}$. Therefore, irrespective of whether $p_{t,i} = 0$ or not, we always have
\begin{equation}\label{eq:unbiased}
\EE[ p_{t,i} \lh_{t,i} | i_{1:t-1} ] = p_{t,i} \xx_{t,i} .
\end{equation}
When $p_{t,i} = 0$, we have $\lh_{t,i} = 0$ but $\xx_{t,i} \le 0$, which means that $\lh_t$ overestimates $\xx_t$ outside the support of $p_t$. Hence, we also have
\begin{equation}\label{eq:overestimate}
\EE[ \lh_t | i_{1:t-1} ] \succeq \xx_t ,
\end{equation}
where $\succeq$ means element-wise greater than.

\begin{framework}
\caption{Gradient-Based Prediction Alg. (GBPA) Template for Multi-Armed Bandits.}
\label{alg:gbpa_bandit}
\begin{algorithmic}
\STATE GBPA$(\tf)$: $\tf$ is a differentiable convex function such that $\gtf \in \D^{N}$ 
\STATE \textbf{Nature:} Adversary  chooses gain vectors $\xx_{t} \in [-1,0]^{N}$ for $t=1,\dots,T$ 
\STATE Learner initializes $\Lh_0 = 0$
\FOR{$t = 1$ to $T$}
\STATE \textbf{Sampling:} Learner chooses $i_t$ according to the distribution $p_t = \gtf (\Lh_{t-1})$
\STATE \textbf{Cost:} Learner incurs (and observes) gain $\xx_{t,i_t} \in [-1,0]$
\STATE \textbf{Estimation:} Learner creates estimate of gain vector $\lh_{t} := \frac{\xx_{t,i_t}}{p_{t,i_t}}\ee_{i_t}$
\STATE \textbf{Update:} Cumulative gain estimate so far $\Lh_{t} = \Lh_{t-1} + \lh_t$
\ENDFOR
\end{algorithmic}
\end{framework}

We now present a basic result bounding the expected regret of GBPA in the multi-armed bandit setting. It is basically just a simple modification of the 
arguments in \cite{Abernethy2015} to deal with the possibility that $p_{t,i} = 0$. We state and prove this result here for completeness without making any
claim of novelty. 
 
\begin{lemma}
\label{lem:genericregret1}
{\bf (Decomposition of the Expected Regret)}
Define the non-smooth potential $\f(\XX) = \max_{i} \XX_i$. The expected
regret of GBPA$(\tf)$ can be written as
\begin{equation}
\label{eq:firstregret}
\EE\Regret_T = \f(\XX_T) - \EE\left[ \sum_{t=1}^{T}\langle p_t, \xx_t \rangle \right] .
\end{equation}
Furthermore, the expected regret of GBPA$(\tf)$ can be bounded by the sum of an
overestimation, an underestimation, and a divergence penalty:
\begin{equation}
\label{eq:genericregret}
\EE\Regret_T \leq
 \underbrace{\tf(0)}_{\text{overestimation penalty}} 
+  \EE\left[ \underbrace{\f(\Lh_T)  - \tf(\Lh_T)}_{\text{underestimation penalty}} \right]
+ \EE\left[ \sum_{t = 1}^{T} \underbrace{\EE[D_{\tf}(\Lh_t, \Lh_{t-1})| i_{1:t-1} ]}_{\text{divergence penalty}} \right] ,
\end{equation}
where the expectations are over the sampling of $i_t$ and $D_{\tf}$ is the Bregman divergence induced by $\tf$.
\end{lemma}
\begin{proof}
First, note that the regret, by definition, is
\[
\Regret_T = \f(\XX_T) - \sum_{t=1}^T \inner{ \ee_{i_t}, \xx_t } .
\]
Under an oblivious adversary, only the summation on the right hand side is random. Moreover $\EE[\inner{\ee_{i_t},\xx_t} | i_{1:t-1} ] = \inner{ p_t , \xx_t}$. This proves the claim in~\eqref{eq:firstregret}.

From~\eqref{eq:unbiased}, we know that $\EE[ \inner{ p_t, \lh_t } | i_{1:t-1} ] =  \inner{ p_t , \xx_t}$ even if some entries in $p_t$ might be zero. Therefore, we have
\begin{equation}\label{eq:intermediateregret}
\EE\Regret_T = \f(\XX_T) - \EE\left[ \sum_{t=1}^T \inner{ p_t, \lh_t } \right].
\end{equation}
From~\eqref{eq:overestimate}, we know that $\XX_T \le \EE[ \Lh_T]$. This implies
\begin{equation}\label{eq:comparator}
\f(\XX_T) \le \f(\EE[ \Lh_T]) \le \EE[\f(\Lh_T)],
\end{equation}
where the first inequality is because $G \succeq G' \Rightarrow \f(G) \ge \f(G')$, and the second inequality is due to the convexity of $\f$.
Plugging~\eqref{eq:comparator} into~\eqref{eq:intermediateregret} yields
\begin{equation}\label{eq:intermediateregret2}
\EE\Regret_T \le  \EE\left[ \f(\Lh_T) - \sum_{t=1}^T \inner{ p_t, \lh_t } \right].
\end{equation}
Now, recalling the definition of Bregman divergence
\[
D_{\tf}(\Lh_t,\Lh_{t-1}) = \tf(\Lh_t) - \tf(\Lh_{t-1}) -  \inner{ \gtf(\Lh_{t-1}), \Lh_t - \Lh_{t-1} }, 
\]
we can write,
\begin{align}
- \sum_{t=1}^T \inner{ p_t, \lh_t }  &= - \sum_{t=1}^T \inner{ \gtf(\Lh_{t-1}), \lh_t } \\
\notag &=  - \sum_{t=1}^T \inner{ \gtf(\Lh_{t-1}), \Lh_t - \Lh_{t-1} } \\
\notag &=  \sum_{t=1}^T \left( D_{\tf}(\Lh_t,\Lh_{t-1}) + \tf(\Lh_{t-1}) - \tf(\Lh_t) \right)\\
&=  \tf(\Lh_{0}) - \tf(\Lh_T) + \sum_{t=1}^T D_{\tf}(\Lh_t,\Lh_{t-1}). \label{eq:bregmanmagic}
\end{align}
The proof ends by plugging ~\eqref{eq:bregmanmagic} into~\eqref{eq:intermediateregret2} and noting that $\tf(\Lh_{0}) = \tf(0)$ is not random.
\end{proof}

\subsection{Stochastic Smoothing of Potential Function} 

Let $\cD$ be a continuous distribution with finite expectation, probability density function $f$, and cumulative distribution function $F$. Consider GBPA with potential function of the form: 
\begin{equation}
	\label{eq:potential_ftpl}
	\tf(\XX; \cD) = \EE_{Z_{1}, \ldots, Z_{N} \stackrel{\text{i.i.d}}{\sim} \cD} \f(\XX + Z),
\end{equation}
which is a \emph{stochastic smoothing} of the non-smooth function $\f(G) = \max_{i} G_{i}$. Note that $Z = (Z_1,\ldots,Z_N) \in \mathbb{R}^N$. We will often hide the dependence on the distribution $\cD$ if the distribution is obvious from the context
or when the dependence on $\cD$ is not of importance in the argument.
Since $\f$ is convex, $\tf$ is also convex. For stochastic smoothing, we have the following result to control the underestimation and overestimation penalty.

\begin{lemma}\label{lem:underover}
For any $G$, we have
\begin{equation}\label{eq:underover}
\f(G) + \EE[Z_1] \le \tf(G) \le \f(G) + EMAX(N)
\end{equation}
where $EMAX(N)$ is any function such that
\[
\EE_{Z_1,\dots,Z_N} [\max_i Z_i] \le EMAX(N).
\]
In particular, this implies that the overestimation penalty $\tf(0)$ is upper bounded by $\f(0) + EMAX(N) = EMAX(N)$ and the underestimation penalty $\f(\Lh_T)  - \tf(\Lh_T)$ is upper bounded by $-\EE[Z_1]$.
\end{lemma}

\begin{proof}
We have,
\begin{align*}
\f(G) + \EE[Z_1] &= \max_i G_i + \EE[Z_i] = \max_i (G_i + \EE[Z_i]) \\
&\le \EE[\max_i (G_i + Z_i)] = \tf(G)\\
&\le \EE[\max_i G_i + \max_i Z_i ] = \max_i G_i + \EE[\max_i Z_i] = \f(G) + \EE[\max_i Z_i].
\end{align*}
Noting that $ \EE[\max_i Z_i]  \le EMAX(N)$ finishes the proof.
\end{proof}

Observe that $\f(G+Z)$ as a function of $G$ is differentiable with probability $1$ (under the randomness of the $Z_i$'s) due to the fact that $Z_i$'s are random variables with a density. By Proposition 2.3 of \cite{bertsekas1973}, we can swap the order of differentiation and expectation:
\begin{equation}
	\label{eq:grad_ftpl}
	\gtf(\XX; \cD) = \EE_{Z_{1}, \ldots, Z_{N} \stackrel{\text{i.i.d}}{\sim} \cD} e_{i^*}, \text{ where } i^* = \argmax_{i = 1, \ldots, N} \{\XX_{i} + Z_{i}\}.
\end{equation}
Note that, for any $\XX$, the random index $i^*$ is unique with probability $1$. Hence, ties between arms can be resolved arbitrarily. It is clear from above that $\gtf$, being an expectation of vectors in the probability simplex, is in the probability simplex.
Thus, it is a valid potential to be used in Framework~\ref{alg:gbpa_bandit}. Now we derive an identity to write the gradient of the smoothed potential function in terms of the expectation of the cumulative distribution function,
\begin{equation}\label{eq:dphi_ftpl}
\begin{aligned}
	\nabla_{i}\tf(\XX) = \frac{\partial \tf}{\partial \XX_i} &= \EE_{Z_1,\ldots,Z_{N}} \ind\{\XX_{i} + Z_{i} > \XX_{j} + Z_{j}, \forall j \neq i\} \\
				&= \EE_{\Lt_{-i}}[\PP_{Z_i}[Z_{i} > \Lt_{-i} - \XX_{i}]] = \EE_{\Lt_{-i}}[1 - F(\Lt_{-i} - \XX_{i})]
\end{aligned}
\end{equation}
where $\Lt_{-i} = \max_{j \ne i} \XX_j + Z_j$. If $\cD$ has unbounded support then this partial derivative is non-zero for all $i$ given any $\XX$. However, it can be zero if $\cD$ has bounded support. Similarly, we have the following useful identity that writes the diagonal of the Hessian of the smoothed potential function in terms of the expectation of the probability density function.
\begin{equation}\label{eq:hessian}
\begin{aligned}
	\nabla^2_{ii}\tf(\XX) &= \frac{\partial}{\partial \XX_i} \nabla_{i}\tf(\XX) 
	 = \frac{\partial}{\partial \XX_i} \EE_{\Lt_{-i}}[1 - F(\Lt_{-i} - \XX_{i})]  \\
	 &=  \EE_{\Lt_{-i}}\left[\frac{\partial}{\partial \XX_i}(1 - F(\Lt_{-i} - \XX_{i}))\right] 
	 =  \EE_{\Lt_{-i}}f(\Lt_{-i} - \XX_{i}).
\end{aligned}
\end{equation}

\subsection{Connection to Follow the Perturbed Leader}\label{sec:ftpl:ftpl}
The sampling step of Framework~\ref{alg:gbpa_bandit} with a stochastically smoothed $\f$ as the potential $\tf$ (Equation~\ref{eq:potential_ftpl}) can be done efficiently. Instead of evaluating the expectation (Equation~\ref{eq:grad_ftpl}), we just take a random sample. Doing so gives us an equivalent of Follow the Perturbed Leader Algorithm (FTPL) \citep{KV-FTL} applied to the bandit setting. On the other hand, the estimation step is hard because generally there is no closed-form expression for $\gtf$. 

To address this issue, \cite{neu2013efficient} proposed Geometric Resampling (GR), an iterative resampling process to estimate $\gtf$ (with bias).
They showed that the extra regret after stopping at $M$ iterations of GR introduces an estimation bias that is at most $\frac{NT}{eM}$ as an additive term.
That is, all GBPA regret bounds that we prove will hold for the corresponding FTPL algorithm that does $M$ iterations of GR at every time step, with an extra additive $\frac{NT}{eM}$ term.
This extra term does not affect the regret rate as long as $M = \sqrt{NT}$, because the lower bound for any adversarial multi-armed bandit algorithm is of the order $\sqrt{NT}$.

\subsection{The Role of the Hazard Rate and Its Limitation}
In previous work, \cite{Abernethy2015} proved that for a continuous random variable $Z$ with finite and nonnegative expectation and support on the whole real line $\RR$, if the hazard rate of the random variable is bounded, i.e,
\[
\sup_z \frac{f(z)}{1-F(z)} < \infty,
\]
then the expected regret of GBPA can be upper bounded as
\[
\EE\Regret_T  = O\Big(\sqrt{NT \times EMAX(N)}\Big) .
\]
Common families of distributions whose regret can be controlled in this way include Gumbel, Frechet, Weibull, Pareto, and Gamma (see \cite{Abernethy2015} for details). However, there are many other families of distributions where the hazard rate condition fails. For example, if the random variable has a bounded support, then the hazard rate would certainly explode at the end of the support. This is, in some sense, an extreme case of violation because the random variable does not even have a tail. There are also some random variables that do have support on $\RR$ but have unbounded hazard rate, e.g. Gaussian, where the hazard rate monotonically increases to infinity. How can we perform analyses of the expected regret of GBPA using those random variables as perturbations? To address these issues, we need to go beyond the hazard rate. 


\section{Perturbations with Bounded Support}

In this section, we prove that GBPA with any continuous distribution that has bounded support and bounded density enjoys sublinear expected regret.
From Lemma \ref{lem:genericregret1} we see that the expected regret can be upper bounded by the sum of three terms. The overestimation penalty can be bounded very easily via Lemma~\ref{lem:underover} for a distribution with bounded support. The underestimation penalty is non-positive as long as the distribution has non-negative expectation. The only term that needs to be controlled with some effort is the divergence penalty.

We first present a general lemma that allows us to write the divergence penalty for a stochastically smoothed potential $\tf$ as a sum involving certain double integrals.

\begin{lemma}\label{lem:divergence}
When using a stochastically smoothed potential as in \eqref{eq:potential_ftpl}, the divergence penalty can be written as
\begin{equation}\label{eq:divbound}
\EE\left[ D_{\tf}(\Lh_t, \Lh_{t-1}) | i_{1:t-1} \right] = 
\sum_{i \in \text{supp}(p_t)} p_{t,i} \int_{0}^{\left|\frac{g_{t,i}}{p_{t,i}}\right|}\EE_{\Lh_{-i}}\left[ \int_{0}^s f(\Lh_{-i} - \Lh_{t-1,i} + r) dr \right] ds 
\end{equation}
where $p_t = \nabla \tf(\Lh_{t-1})$, $\Lh_{-i} = \max_{j \ne i} \Lh_{t-1,j} + Z_j$ and $\text{supp}(p_t) = \{ i \::\: p_{t,i} > 0 \}$.
\end{lemma}

\begin{proof}
To reduce clutter, we drop the time subscripts: we use $\Lh$ to denote the cumulative estimate $\Lh_{t-1}$, $\lh$ to denote the marginal estimate $\lh_t = \Lh_{t} - \Lh_{t - 1}$, $p$ to denote $p_t$, and $\xx$ to denote the true gain $\xx_t$.
Note that by definition of Framework~\ref{alg:gbpa_bandit}, $\lh$ is a sparse vector with one non-zero and non-positive coordinate $\lh_{i_{t}} = \xx_{i_t} / p_{i_t} = -\left| \xx_{i_t} / p_{i_t} \right|$. Morever, conditioned on $i_{1:t-1}$, $i_t$ takes value
$i$ with probability $p_i$. For any $i \in \text{supp}(p)$, let
\[
h_i(r) =  D_\tf(\Lh - r \ee_{i}, \Lh),
\]
so that $h_i'(r) = -\nabla_{i} \tf\left( \Lh - r \ee_{i} \right) + \nabla_{i} \tf\left( \Lh \right)$ and $h_i''(r) = \nabla^2_{ii} \tf\left( \Lh - r \ee_{i} \right)$.
Now we write:
\begin{align*}
\EE[D_{\tf}(\Lh + \lh, \Lh)| i_{1:t-1} ] &= \sum_{i \in \text{supp}(p)} p_i D_{\tf}(\Lh + \xx_i/p_i \ee_i,\Lh) = \sum_{i \in \text{supp}(p)} p_i D_{\tf}(\Lh - \left| \xx_i/p_i \right| \ee_i,\Lh) \\
&= \sum_{i \in \text{supp}(p)} p_i h_i(\left|g_{i}/p_{i}\right|) = \sum_{i \in \text{supp}(p)} p_i \int_{0}^{\left|g_{i}/p_{i}\right|}\int_{0}^s h_i''(r)dr \, ds \\
&= \sum_{i \in \text{supp}(p)} p_i \int_{0}^{\left|g_{i}/p_{i}\right|}\int_{0}^s \nabla^2_{ii} \tf\left( \Lh - r \ee_{i} \right) dr \, ds \\
&= \sum_{i \in \text{supp}(p)} p_i \int_{0}^{\left|g_{i}/p_{i}\right|}\int_{0}^s \EE_{\Lh_{-i}} f(\Lh_{-i} - \Lh_i + r) dr \, ds \\
&= \sum_{i \in \text{supp}(p_t)} p_{t,i} \int_{0}^{\left|g_{i}/p_{i}\right|} \EE_{\Lh_{-i}}\left[ \int_{0}^s f(\Lh_{-i} - \Lh_{i} + r) dr \right] ds. 
\end{align*}
The second equality on the first line implicitly used the assumption that $g_i \le 0$, i.e, the ``gains'' are non-positive.
The second equality on the second line used that $h_i(0) = 0$, and the equality on the fourth line used Equation \eqref{eq:hessian}.
\end{proof}

Note that each summand in the divergence penalty expression above involves an integral of the density function of the distribution $\cD$ over an interval. The main idea to control the divergence penalty for a bounded support distribution is to truncate the interval at the end of the support. For points that are close to the end of the support, we bound the integral by the product of the bound on the density and the interval length. For points that are far from the end of the support, we bound the integral through the hazard rate as was done by \cite{Abernethy2015}.

For a general continuous random variable $Z$ with bounded density, bounded support, we first shift it (which obviously does not change the distribution of the random action choice $i_t$ and hence the expected regret) and scale it so that the support is a subset of $[0,1]$ with $\sup\{z:F(z) = 0\} = 0$ and $\inf\{z:F(z) = 1\} = 1$ where $F$ denotes the CDF of $Z$. A benefit of this normalization is that the expectation of the random variable becomes non-negative so the underestimation penalty is guaranteed to be non-positive. After scaling, we assume that the bound on the density is $L$. We consider the perturbation $\eta Z$ where $\eta > 0$ is a tuning parameter. Write $F_\eta(x)$ and $f_\eta(x)$ to denote the CDF and PDF of the scaled random variable $\eta Z$ respectively. If $F$ is strictly increasing, we know that $F^{-1}$ exists. If not, define $F^{-1}(y) = \inf\{z: F(z) = y\}$. Elementary calculation gives the following useful facts:
$$
F_\eta(z) = F(\frac{z}{\eta}), f_\eta(z) = \frac{f(\frac{z}{\eta})}{\eta}, F^{-1}_\eta(y) = \eta F^{-1}(y).
$$
\begin{theorem}\label{thm:bounded_dist_regret}
{\bf (Divergence Penalty Control, Bounded Support)}
The divergence penalty in the GBPA regret bound using the scaled perturbation $\eta Z$, where $Z$ is drawn from a bounded support distribution satisfying the conditions above, can be upper bounded, for any $\epsilon > 0$, by
$$
NL \Big(\frac{1}{2 \eta \epsilon} + 1 - F^{-1}(1-\epsilon) \Big).
$$
\end{theorem}

\begin{proof}
From Lemma~\ref{lem:divergence}, we have, with $\Lh_{-i} = \max_{j \ne i} \Lh_{t-1,j} + \eta Z_j$,
\begin{align}
\notag &\quad \EE\left[ D_{\tf}(\Lh_t, \Lh_{t-1}) | i_{1:t-1} \right] \\
\notag &= \sum_{i \in \text{supp}(p_t)} p_{t,i} \int_{0}^{\left|\frac{g_{t,i}}{p_{t,i}}\right|}\EE_{\Lh_{-i}}\left[ \int_{0}^s f_\eta(\Lh_{-i} - \Lh_{t-1,i} + r) dr \right] ds \\
\notag & = \sum_{i \in \text{supp}(p_t)} p_{t,i} \int_{0}^{\left|\frac{g_{t,i}}{p_{t,i}}\right|}\EE_{\Lh_{-i}}\left[ \int_{\Lh_{-i} - \Lh_{t-1,i}}^{\Lh_{-i} - \Lh_{t-1,i} + s} f_\eta(z)dz \right] ds  \\
\notag &\le \sum_{i \in \text{supp}(p_t)} p_{t,i} \int_{0}^{\left|\frac{g_{t,i}}{p_{t,i}}\right|} \Big( \EE_{\Lh_{-i}}
\notag \Bigg[  \underbrace{ \int_{[\Lh_{-i} - \Lh_{t-1,i},\Lh_{-i} - \Lh_{t-1,i} + s] \backslash [F_\eta^{-1}(1-\epsilon),\eta]} f_\eta(z)dz }_{(I)} \Bigg]  \\
&\quad \quad + \underbrace{ \int_{[F_\eta^{-1}(1-\epsilon),\eta]} f_\eta(z)dz }_{(II)} \Big) ds. \label{eq:twointegrals}
\end{align}

We bound the two integrals above differently. For the first integral, we add the restriction $f_\eta(z) > 0$ by intersecting the integral interval with the support of the function $f_\eta(z)$, denoted as $I_{f_\eta(z)}$ so that $1-F_\eta(z)$ is not $0$ on the interval to be integrated. Thus, we get,
\begin{align}
\notag (I) &= \int_{([\Lh_{-i} - \Lh_{t-1,i},\Lh_{-i} - \Lh_{t-1,i} + s] \backslash [F_\eta^{-1}(1-\epsilon),\eta]) \cap I_{f_\eta(z)}} f_\eta(z)dz \\
\notag &= \int_{([\Lh_{-i} - \Lh_{t-1,i},\Lh_{-i} - \Lh_{t-1,i} + s] \backslash [F_\eta^{-1}(1-\epsilon),\eta]) \cap I_{f_\eta(z)}} (1-F_\eta(z)) \cdot \frac{f_\eta(z)}{1-F_\eta(z)} dz \\
\notag &\le \int_{([\Lh_{-i} - \Lh_{t-1,i},\Lh_{-i} - \Lh_{t-1,i} + s] \backslash [F_\eta^{-1}(1-\epsilon),\eta]) \cap I_{f_\eta(z)}} (1-F_\eta(z)) \cdot \frac{L}{\eta \epsilon} dz\\
&\le (1-F_\eta(\Lh_{-i} - \Lh_{t-1,i}) ) \frac{sL}{\eta \epsilon} . \label{eq:integral1}
\end{align}
The first inequality holds because $f_\eta(z) \le L/\eta$ and $(1-F_\eta(z)) \ge \epsilon$ on the set of $z$'s over which we are integrating.
The second inequality holds because on the set under consideration $1-F_\eta(z) \le 1-F_\eta(\Lh_{-i} - \Lh_{t-1,i})$ and the measure of the set is at most $s$.

For the second integral, we use the bound $f_\eta(z) \le L/\eta$ again to get,
\begin{equation}\label{eq:integral2}
(II) = \int_{[F_\eta^{-1}(1-\epsilon),\eta]} f_\eta(z)dz \le \frac{L}{\eta} \cdot (\eta-F_\eta^{-1}(1-\epsilon)) .
\end{equation}

Plugging ~\eqref{eq:integral1} and~\eqref{eq:integral2} into~\eqref{eq:twointegrals}, we can bound the divergence penalty by,
\begin{align*}
&\le  \sum_{i \in \text{supp}(p_t)} p_{t,i} \int_{0}^{\left|\frac{g_{t,i}}{p_{t,i}}\right|} \Big( \EE_{\Lh_{-i}} [1-F_\eta(\Lh_{-i} - \Lh_{t-1,i}) ] \frac{sL}{\eta \epsilon} +  \frac{L(\eta - F_\eta^{-1}(1-\epsilon))}{\eta} \Big) ds \\
&= \sum_{i \in \text{supp}(p_t)} p_{t,i} \int_{0}^{\left|\frac{g_{t,i}}{p_{t,i}}\right|} \Big( p_{t,i} \frac{sL}{\eta \epsilon} +  L(1 - F^{-1}(1-\epsilon)) \Big) ds\\
&= \sum_{i \in \text{supp}(p_t)} p_{t,i}  \Big( p_{t,i} \frac{L}{\eta \epsilon} \frac{g_{t,i}^2}{2p_{t,i}^2} +  L(1 - F^{-1}(1-\epsilon)) \frac{|g_{t,i}|}{p_{t,i}} \Big) \\
&\le \sum_{i \in \text{supp}(p_t)} \Big( \frac{L}{2 \eta \epsilon}  +  L(1 - F^{-1}(1-\epsilon)) \Big) \\
&\le NL \Big(\frac{1}{2 \eta \epsilon} + 1 - F^{-1}(1-\epsilon) \Big).
\end{align*}
The second to last inequality holds because $|g_{t,i}| \le 1$ and the last inequality holds because the sum over $i$ is at most over all $N$ arms.
\end{proof}

The regret bound for the uniform distribution is now an easy corollary.
\begin{corollary}\label{cor:uniform}
{\bf (Regret Bound for Uniform)}
For GBPA run with a stochastically smoothed potential using an appropriately scaled $[0,1]$ uniform perturbation where $\eta = (NT)^{2/3}$, the expected regret can be upper bounded by
$3(NT)^{2/3}$.
\end{corollary}

\begin{proof}
For $[0,1]$ uniform distribution, we have $L = 1$, $F^{-1}(1-\epsilon) = 1-\epsilon$ so the divergence penalty is upper bounded by 
$$
NT(\frac{1}{2\eta \epsilon} + \epsilon).
$$
If we let $\epsilon = \frac{1}{\sqrt{2\eta}}$, we can see that the divergence penalty is upper bounded by $NT\sqrt{\frac{2}{\eta}}$. Together with the overestimation penalty which is trivially bounded by $\eta$ and a non-positive underestimation penalty, we see that the final regret bound is 
$$
NT\sqrt{\frac{2}{\eta}} + \eta.
$$
Setting $\eta = (NT)^{2/3}$ gives the desired result.
\end{proof}

For a general perturbation with bounded support and bounded density, the rate at which $1-F^{-1}(1-\epsilon)$ goes to $0$ as $\epsilon \to 0$ can vary but we can always guarantee sublinear expected regret.
\begin{corollary}\label{cor:bounded}
{\bf (Asymptotic Regret Bound for Bounded Support)}
For stochastically smoothed GBPA using general continuous random variable $\eta Z$ where $Z$  has bounded density and bounded support contained in $[0,1]$ and $\eta = (NT)^{2/3}$, the expected regret grows sublinearly, i.e.,
$$
\lim_{T \rightarrow \infty} \frac{\EE \Regret_T }{T} = 0.
$$
\end{corollary}

\begin{proof}
For a general distribution, let $\epsilon = \frac{1}{\sqrt{\eta}}$. Since the overestimation penalty is trivially bounded by $\eta$ and the underestimation penalty is non-positive, the expected regret can be upper bounded by  
$$
LNT \Big(\frac{1}{2\sqrt{\eta}} + 1 - F^{-1}(1-\frac{1}{\sqrt{\eta}}) \Big) + \eta.
$$
Setting $\eta = (NT)^{2/3}$ we see that the expected regret can be upper bounded by
$$
(\frac{L}{2} + 1)(NT)^{2/3} + LNT(1 - F^{-1}(1-\frac{1}{\sqrt{\eta}}).
$$
Since 
$$
\lim_{T \rightarrow \infty} 1 - F^{-1}(1-\frac{1}{\sqrt{\eta}}) = \lim_{\eta \rightarrow \infty} 1 - F^{-1}(1-\frac{1}{\sqrt{\eta}}) = 1- F^{-1}(1) = 0 ,
$$
we conclude that
$$
\lim_{T \rightarrow \infty} \frac{\EE \Regret_T }{T} = 0.
$$
\end{proof}


\section{Perturbations with Unbounded Support}
Unlike perturbations with bounded support, perturbations with unbounded support (on the right) do have non-zero right tail probabilities, ensuring that $p_{t,i} > 0$ always. However, the tail behavior may be such that the hazard rate is unbounded. Still, under mild assumptions, perturbations with unbounded support (on the right) can also be shown to have near optimal expected regret in $T$, using the notion of \emph{generalized hazard rate} that we now introduce.

\subsection{Generalized Hazard Rate}\label{sec:ghr}
We already know how to control the underestimation and overestimation penalties via Lemma~\ref{lem:underover}. So our main focus will be to control the divergence penalty. Towards this end,
we define the generalized hazard rate for a continuous random variable $Z$ with support unbounded on the right, parameterized by $\alpha \in [0,1)$, as 
\begin{equation}\label{eq:genhazard}
h_{\alpha}(z) := \frac{f(z)|z|^\alpha}{(1-F(z))^{1-\alpha}},
\end{equation}
where $f(z)$ and $F(z)$ denotes the PDF and CDF of $Z$ respectively. Note that by setting $\alpha = 0$ we recover the standard hazard rate.

One of the main results of this paper is the following. Note that it includes the result (Lemma 4.3) of \cite{Abernethy2015} as a special case. 

\begin{theorem}\label{thm:divergence}
{\bf (Divergence Penalty Control via Generalized Hazard Rate)}
Let $\alpha \in [0,1)$. Suppose we have $\forall z \in \RR, h_{\alpha}(z) \le C $.
Then,
\begin{equation*}\label{eq:divergence}
\EE[D_{\tf}(\Lh_t, \Lh_{t-1}) | i_{1:t-1} ] \le \frac{2C}{1-\alpha} \times N.
\end{equation*}
\end{theorem}

\begin{proof}
Because of the unbounded support of $Z$, $\text{supp}(p_t) = \{1,\ldots,N\}$. Lemma~\ref{lem:divergence} gives us:
\begin{align*}
\EE[D_{\tf}(\Lh_t, \Lh_{t-1})|i_{1:t-1}] &= \sum_{i=1}^N p_{t,i} \int_{0}^{\left| g_{t,i}/p_{t,i}\right|} \EE_{\Lt_{-i}} \int_{0}^s f(\Lt_{-i} - \Lh_{t-1,i} + r) drds \\
&= \sum_{i=1}^N p_{t,i} \int_{0}^{\left| g_{t,i}/p_{t,i}\right|} \EE_{\Lt_{-i}} \int_{\Lt_{-i} - \Lh_{t-1,i}}^{\Lt_{-i} - \Lh_{t-1,i} + s} f(z) dz ds  \\
&\le C \sum_{i=1}^N p_{t,i} \int_{0}^{\left| g_{t,i}/p_{t,i}\right|} \EE_{\Lt_{-i}} \int_{\Lt_{-i} - \Lh_{t-1,i}}^{\Lt_{-i} - \Lh_{t-1,i} + s} (1-F(z))^{1-\alpha}|z|^{-\alpha}dz \, ds \\
&\le C \sum_{i=1}^N p_{t,i} \int_{0}^{\left| g_{t,i}/p_{t,i}\right|} \EE_{\Lt_{-i}}  (1-F(\Lt_{-i} - \Lh_{t-1,i}))^{1-\alpha} \int_{\Lt_{-i} - \Lh_{t-1,i}}^{\Lt_{-i} - \Lh_{t-1,i} + s} |z|^{-\alpha}dz \, ds.
\end{align*}
Since the function $|z|^{-\alpha}$ is symmetric in $z$, monotonically decreasing as $|z| \rightarrow \infty$, we have
$$
\int_{\Lt_{-i} - \Lh_{t-1,i}}^{\Lt_{-i} - \Lh_{t-1,i} + s} |z|^{-\alpha}dz \le 
\int_{-s/2}^{s/2} |z|^{-\alpha} dz = \frac{2^{\alpha}}{1-\alpha}s^{1-\alpha}.
$$
Also, note that $z^{1-\alpha}$ is a concave function in $z$. Hence, by Jensen's inequality, 
$$
\EE_{\Lt_{-i}} [(1-F(\Lt_{-i} - \Lh_{t-1,i}))^{1-\alpha}]
\le
(\EE_{\Lt_{-i}}[1-F(\Lt_{-i} - \Lh_{t-1,i}])^{1-\alpha} = p_{t,i}^{1-\alpha}.
$$
Therefore, 
\begin{align*}
\EE[D_{\tf}(\Lh_t, \Lh_{t-1})|i_{1:t-1}]
&\le \frac{2^{\alpha} C}{1-\alpha} \sum_{i=1}^N p_{t,i} \int_{0}^{\left| g_{t,i}/p_{t,i}\right|} p_{t,i}^{1-\alpha} s^{1-\alpha}ds \\
&= \frac{2^{\alpha} C}{1-\alpha} \sum_{i=1}^N p_{t,i}^{2-\alpha} \int_{0}^{\left| g_{t,i}/p_{t,i}\right|} s^{1-\alpha} ds \\
&= \frac{2^{\alpha} C}{(1-\alpha)(2-\alpha)} \sum_{i=1}^N p_{t,i}^{2-\alpha} \left| g_{t,i}/p_{t,i}\right|^{2-\alpha} \\
&= \frac{2^{\alpha} C}{(1-\alpha)(2-\alpha)} \sum_{i=1}^N |g_{t,i}|^{2-\alpha} \\
&\le \frac{2^{\alpha} C}{(1-\alpha)(2-\alpha)} N \le \frac{2C}{1-\alpha} N.
\end{align*}
\end{proof}

A regret bound now easily follows.
\begin{theorem}\label{thm:master}
{\bf (Regret Bound via Generalized Hazard Rate)}
Suppose we use a stochastically smoothed GBPA with perturbation $\eta Z$, with $Z$'s generalized hazard rate being bounded:
$h_\alpha(x) \le C, \forall x \in \RR$ for some $\alpha \in [0,1)$, and 
$$
\EE_{Z_1,\dots,Z_N}[\max\limits_i Z_i] - \EE[Z_1] \le Q(N),
$$
where $Q(N)$ is some function of $N$. Then,  if we set $\eta = (\frac{2CNT}{(1-\alpha) Q(N)})^{1/(2-\alpha)}$, the expected regret of GBPA is no greater than
$$
2 \times (\frac{2C}{1-\alpha})^{1/(2-\alpha)} \times (NT)^{1/(2-\alpha)} \times Q(N)^{(1-\alpha)/(2-\alpha)}.
$$
In particular, this implies that the algorithm has sublinear expected regret.
\end{theorem}

\begin{proof}
The divergence penalty can be controlled through Theorem \ref{thm:divergence} once we have bounded generalized hazard rate. It remains to control the overestimation and underestimation penalty. By Lemma~\ref{lem:underover}, they are at most $\EE_{Z_1,\dots,Z_n} [\max\limits_i Z_i]$ and $-\EE[Z_1]$ respectively. Suppose we scale the perturbation $Z$ by $\eta > 0$, i.e., we add $\eta Z_{i}$ to each coordinate. It is easy to see that $\EE[ \max_{i=1,\ldots,n} \eta Z_i] = \eta \EE[ \max_{i=1,\ldots,n} Z_i]$ and $\EE[\eta Z_1] = \eta \EE[Z_1]$. For the divergence penalty, observe that $F_\eta(t) = F(t/\eta)$ and thus $f_\eta(t) = \frac{1}{\eta}f(t/\eta)$. Hence, the bound on the generalized hazard rate for perturbation $\eta Z$ is $\eta^{\alpha-1} C$. Plugging new bounds for the scaled perturbations into Lemma \ref{lem:genericregret1} gives us 
\begin{equation*}
\EE\Regret_T \le \eta^{\alpha-1} \frac{2C}{1-\alpha} \times NT + \eta Q(N). 
\end{equation*}
Setting $\eta = (\frac{2CNT}{(1-\alpha) Q(N)})^{1/(2-\alpha)}$ finishes the proof.
\end{proof}


\subsection{Gaussian Perturbation}
In this section we prove that GBPA with the standard Gaussian perturbation incurs a near optimal expected regret in both $N$ and $T$. Let $F(z)$ and $f(z)$ denote the CDF and PDF of standard Gaussian distribution.

\begin{lemma}[\cite{Baricz2008}]\label{lemma:gaussian1}
For standard Gaussian random variable, we have
$$
z < \frac{f(z)}{1-F(z)} < \frac{z}{2} + \frac{\sqrt{z^2+4}}{2}. 
$$
\end{lemma}

This lemma together with example 2.6 in \cite{Thomas1971} show that the hazard rate of a standard Gaussian random variable increases monotonically to infinity.
However, we can still bound the generalized hazard rate for strictly positive $\alpha$.

\begin{lemma}\label{lemma:gaussian2}
{\bf (Generalized Hazard Bound for Gaussian)}
For any $\alpha \in (0,1)$, we have
$$
\frac{f(z)|z|^{\alpha}}{(1-F(z))^{1-\alpha}} \le \frac{2}{\alpha}.
$$
\end{lemma}
The proof of this lemma is deferred to the appendix.

The bounded generalized hazard rate shown in the above lemma can be used to control the divergence penalty. Combined with other knowledge of the standard Gaussian random variable we are able to give a bound on the expected regret.
\begin{corollary}\label{cor:gaussian}
The expected regret of GBPA with an appropriately scaled standard Gaussian random variable as perturbation where $\eta = \left(\frac{4NT}{\alpha(1-\alpha)\sqrt{2\log N}}\right)^{1/(2-\alpha)}$ has an expected regret at most
$$
2(C_1C_2NT)^{1/(2-\alpha)}(\sqrt{2\log N})^{(1-\alpha)/(2-\alpha)}
$$
where $C_1 = \frac{2}{\alpha}$, $C_2 = \frac{2}{1-\alpha}$, for any $\alpha \in (0,1)$.
\end{corollary}
\begin{proof}
It is known that for standard Gaussian random variable, we have $\EE[Z_1] = 0$ and 
$$
\mathbb{E}_{Z_1,\dots,Z_n}[\max\limits_i Z_i] \le \sqrt{2\log N}.
$$ Plug in to Theorem \ref{thm:master} gives the result.
\end{proof}

It remains to optimally tune $\alpha$ in the above bound. 

\begin{theorem}\label{thm:gaussian}
{\bf (Regret Bound for Gaussian)}
The expected regret of GBPA with an appropriately scaled standard Gaussian random variable as perturbation where $\eta = \left(\frac{4NT}{\alpha(1-\alpha)\sqrt{2\log N}}\right)^{1/(2-\alpha)}$ and $\alpha = \frac{1}{\log T}$ has an expected regret at most
$$
96 \sqrt{NT} \times N^{1/\log T}\sqrt{\log N}\log T
$$
for $T > 4$. If we assume that $T > N$, the expected regret can be upper bounded by
$$
278 \sqrt{NT} \times \sqrt{\log N} \log T.
$$
\end{theorem}
The proof of this theorem is also deferred to the appendix.

\subsection{Sufficient Condition for Near Optimal Regret}
In Section \ref{sec:ghr} we showed that if the generalized hazard rate of a distribution is bounded, the expected regret of the GBPA can be controlled. 
In this section, we are going to prove that under reasonable assumptions on the distribution of the perturbation, the FTPL enjoys near optimal expected regret. Note that most proofs in this section are deferred to the appendix.

{\bf Assumptions (a)-(c).} Before we proceed, let us formally state our assumptions on the distributions we will consider. 
The distribution needs to (a) be continuous and have bounded density (b) have finite expectation (c) have support unbounded in the $+\infty$ direction. 

Note that if the expectation of the random perturbation is negative, we shift it so that the expectation is zero. Hence the underestimation penalty is non-positive.
In addition to the assumptions we have made above, we make another assumption on the eventual monotonicity of the hazard rate.
$$
\text{\bf Assumption (d)} \quad h_0(z) = \frac{f(z)}{1-F(z)} \text{ is eventually monotone}.
$$
``Eventually monotone" means that $\exists z_0 \ge 0$ such that if $z > z_0$, $\frac{f(z)}{1-F(z)}$ is non-decreasing or non-increasing. This assumption might appear hard to check, but numerous theorems are available to establish the monotonicity of hazard rate, which is much stronger than what we are assuming here. For example, see Theorem 2.4 in \cite{Thomas1971}, Theorem 2 and Theorem 4 in \cite{Chechile2003}, \cite{Chechile2009}. In fact, most natural distributions do satisfy this assumption \citep{Bagnoli2005}. 

Before we proceed, we mention a standard classification of random variables into two classes based on their tail property. 
\begin{definition}[see, for example, \cite{Foss2009}]
A function $f(z) \ge 0$ is said to be heavy-tailed if and only if
$$
\lim_{z \rightarrow \infty}\sup f(z)e^{\lambda z} = \infty \quad \text{ for all } \lambda > 0.
$$
A distribution with CDF $F(z)$ and $\overline{F}(z) = 1-F(z)$ is said to be heavy-tailed if and only if $\overline{F}(z)$ is heavy-tailed. If the distribution is not heavy-tailed, we say that it is light-tailed.
\end{definition}
It turns out that under assumptions (a)-(d), if the distribution is also heavy-tailed, then the hazard rate itself is bounded. If the distribution is light-tailed, we need an additional assumption on the eventual monotonicity of a function similar to the generalized hazard rate to ensure the boundedness of the generalized hazard rate. But before we state and prove the main results, we introduce some functions and prove an intermediate lemma that will be useful to prove the main results.

Define $R(z) = -\log \overline{F}(z)$ so that we have $\overline{F}(z) = e^{-R(x)}$ and $R'(z)= \frac{f(z)}{\overline{F}(z)} = h_0(z)$.
\begin{lemma}\label{lemma:exp_monotone}
Under assumptions (a)-(d), we have
$$
\overline{F}(z)e^{\lambda z} \text{ is eventually monotone } \forall \lambda > 0.
$$
\end{lemma}

\begin{proof}
Let $g(z) = \overline{F}(z)e^{\lambda z}$, then $g'(z) = e^{\lambda z}\overline{F}(z) (\lambda-\frac{f(z)}{\overline{F}(z)})$. Since $\frac{f(z)}{\overline{F}(z)} \text{ is eventually monotone}$ by assumption (d), $g'(z)$ is eventually positive, negative or zero. The lemma immediately follows.
\end{proof}

We are finally ready to present the main results in this section.
\begin{theorem}\label{thm:gen_hazard_bound_heavy}
{\bf (Heavy Tail Implies Bounded Hazard)}
Under assumptions (a) - (d), if the distribution is also heavy-tailed, then the hazard rate is bounded, i.e,
$$
\sup_z \frac{f(z)}{\overline{F}(z)} < \infty.
$$
\end{theorem}

Unlike heavy-tailed distributions, the hazard rate of light-tailed distributions might be unbounded. However, it turns out that if we make an additional assumption on the eventual monotonicity of a function similar to the generalized hazard rate, we can still guarantee the boundedness of the generalized hazard rate. 
$$
\text{\bf Assumption (e)} \quad \exists \delta \in (0,1] \text{ such that }  \frac{f(z)}{(1-F(z))^{1-\delta}} \text{ is eventually monotone}.
$$
\begin{theorem}\label{thm:gen_hazard_bound_light}
{\bf (Light Tail Implies Bounded Generalized Hazard)}
Under assumptions (a) - (e), if the distribution is also light-tailed, then for any $\alpha \in (\delta,1)$, the generalized hazard rate $h_\alpha(z)$ is bounded, i.e,
$$
\sup_z \frac{f(z)|z|^\alpha}{(\overline{F}(z))^{1-\alpha}} < \infty.
$$
\end{theorem}
Combining the above result with control of the divergence penalty gives us the following corollary.
\begin{corollary}\label{cor:lighttailed}
Under assumptions (a)-(e), if the distribution is also light-tailed, the expected regret of GBPA with appropriately scaled perturbations drawn from that distribution is, for all $\alpha \in (\delta,1)$ and $\xi > 0$, 
$$
O\Big((TN)^{1/(2-\alpha)}N^\xi\Big).
$$
In particular, if assumption (e) holds for all $\delta \in (0,1)$, then the expected regret of GBPA is $O\Big((TN)^{1/2+\epsilon}\Big)$ for all $\epsilon > 0$, i.e, it is near optimal in both $N$ and $T$.
\end{corollary}
Next we consider a family of light-tailed distributions that do not have a bounded hazard rate.
\begin{definition}
The exponential power (or generalized normal) family of distributions, denoted as $\cD_\beta$ where $\beta >1$, is defined via the cdf 
$$
f_\beta(z) = C_\beta e^{-z^\beta}, \quad z\ge 0.
$$
\end{definition}
The next theorem shows that GBPA with perturbations from this family of distributions enjoys near optimal expected regret in both $N$ and $T$. 
\begin{theorem}\label{thm:exp_power}
{\bf (Regret Bound for Exponential Power Family)}
$\forall \beta > 1$, the expected regret of GBPA with appropriately sclaed perturbations drawn from $\cD_\beta$ is, for all $\epsilon > 0$, $O\Big((TN)^{1/2+\epsilon}\Big)$.
\end{theorem}

\section{Conclusion and Future Work}

Previous work on providing regret guarantees for FTPL algorithms in the adversarial multi-armed bandit setting required a bounded hazard rate condition.
We have shown how to go beyond the hazard rate condition but a number of questions remain open. For example, what if we use FTPL with perturbations from discrete distributions such as Bernoulli distribution? In the full information setting \cite{devroye2013prediction} and \cite{van2014follow} have considered random walk perturbation and dropout perturbation, both leading to minimax optimal regret. But to the best of our knowledge those distributions have not been analyzed in the adversarial multi-armed bandit problem.

An unsatisfactory aspect of even the tightest bounds for FTPL algorithms from existing work, including ours, is that they never reach
the minimax optimal $O(\sqrt{N T})$ bound. They come very close to it: up to logarithmic factors. It is known that FTRL algorithms, using the negative Tsallis entropy
as the regularizer, can achieve the optimal bound \citep{audibert2009minimax,audibert2011minimax,Abernethy2015}. Is there a perturbation that can achieve the optimal bound?

We only considered multi-armed bandits in this work. There has been some interest in using FTPL algorithms for combinatorial bandit problems (see, for example, \cite{neu2013efficient}).
In future work, it will be interesting to extend our analysis to combinatorial bandit problems.

\paragraph*{Acknowledgments.}
We thank Jacob Abernethy and Chansoo Lee for helpful discussions. We acknowledge the support of NSF CAREER grant IIS-1452099 and a Sloan Research Fellowship.

\bibliographystyle{plainnat}
\bibliography{bandit_beyond_hazard}


\appendix

\section{Proofs}
\subsection{Proof of Lemma \ref{lemma:gaussian2}}
\begin{proof}
Since the numerator of the left hand side is an even function of $z$, and the denominator is a decreasing function, and the inequality is trivially true when $z = 0$, it suffices to prove for $z > 0$, which we assume for the rest of the proof. From Lemma \ref{lemma:gaussian1} we can derive that 
$$
\frac{f(z)}{1-F(z)} < z+1.
$$
Therefore,
\begin{align*}
\frac{f(z)|z|^{\alpha}}{(1-F(z))^{1-\alpha}} &\le \frac{f(z)z^{\alpha}}{(\frac{f(z)}{z+1})^{1-\alpha}} = (f(z)z)^{\alpha}(z+1)^{1-\alpha} \\
&\le f(z)^{\alpha}(z+1) \le zf(z)^{\alpha} + 1 = \sqrt{\frac{1}{2\pi}}ze^{-\alpha z^2/2} + 1.
\end{align*}
Let $g(z) = ze^{-\alpha z^2/2}$, $g'(z) = (1-\alpha z^2)e^{-\alpha z^2/2}$. Therefore $g(z)$ is maximized at $z^* = \sqrt{\frac{1}{\alpha}}$. Therefore, 
\begin{align*}
\frac{f(z)|z|^{\alpha}}{(1-F(z))^{1-\alpha}} &\le \sqrt{\frac{1}{2\pi}}ze^{-\alpha z^2/2} + 1 \le \sqrt{\frac{1}{2\pi}}z^*+1 \le z^* + 1= \sqrt{\frac{1}{\alpha}} + 1 \le \frac{2}{\alpha}.
\end{align*}
\end{proof}

\subsection{Proof of Theorem \ref{thm:gaussian}}
\begin{proof}
From Corollary \ref{cor:gaussian} we see that the expected regret can be upper bounded by
$$
2(C_1C_2NT)^{1/(2-\alpha)}(\sqrt{2\log N})^{(1-\alpha)/(2-\alpha)}
$$
where $C_1 = \frac{2}{\alpha}$ and $C_1 = \frac{2}{1-\alpha}$. Note that
\begin{align*}
&2(C_1C_2NT)^{1/(2-\alpha)}(\sqrt{2\log N})^{(1-\alpha)/(2-\alpha)} \\
\le &4 (C_1C_2)^{1/(2-\alpha)}N^{1/(2-\alpha)}\sqrt{\log N}^{(1-\alpha)/(2-\alpha)}T^{1/(2-\alpha)} \\
= &4N^{1/(2-\alpha)}\sqrt{\log N}^{(1-\alpha)/(2-\alpha)}T^{1/2} \times (C_1C_2)^{1/(2-\alpha)} T^{\alpha/(4-2\alpha)} \\
\le & 4N^{1/2}N^{\alpha/(4-2\alpha)}\sqrt{\log N}T^{1/2} \times (\frac{4}{\alpha(1-\alpha)})^{1/(2-\alpha)}T^{\alpha/(4-2\alpha)} \\
\le & 4N^{1/2}N^{\alpha}\sqrt{\log N}T^{1/2} \times \frac{4T^{\alpha}}{\alpha(1-\alpha)} \\
\le & 16\sqrt{NT}N^{\alpha}\sqrt{\log N}\times \frac{T^{\alpha}}{\alpha(1-\alpha)}.
\end{align*}
\end{proof}
If we let $\alpha = \frac{1}{\log T}$, then $T^\alpha = T^{1/\log T} = e < 3$. Then, we have, for $T > 4$,
$$
\frac{T^{\alpha}}{\alpha(1-\alpha)} \le \frac{3 \log T }{1-\frac{1}{\log T}} = \frac{3\log^2 T }{\log T-1} \le 6\log T.
$$
Putting things together finishes the proof.

\subsection{Proof of Theorem \ref{thm:gen_hazard_bound_heavy}}
\begin{proof}
If the distribution is heavy-tailed, we have
$$
\lim_{z \rightarrow \infty}\sup \overline{F}(z)e^{\lambda z} = \infty \quad \text{ for all } \lambda > 0.
$$
By Lemma \ref{lemma:exp_monotone}, we can erase the supremum operator and just write 
$$
\lim_{z \rightarrow \infty} \overline{F}(z)e^{\lambda z} = \infty \quad \text{ for all } \lambda > 0.
$$
Hence,
$$
\lim_{z \rightarrow \infty}\overline{F}(z)e^{\lambda z} = \lim_{x \rightarrow \infty} e^{-R(z) + \lambda z} = \infty \text{ for all } \lambda > 0 \Rightarrow \lim\sup_{z \rightarrow \infty} \frac{R(z)}{z} = 0.
$$
Note that $R'(z) = \frac{f(z)}{\overline{F}(z)}$, which is eventually monotone by assumption. Therefore, we can conclude that
\begin{equation*}
\lim\sup_{z \rightarrow \infty} R'(z) <\infty \Rightarrow \sup_z \frac{f(z)}{\overline{F}(z)} < \infty .
\end{equation*}
\end{proof}

\subsection{Proof of Theorem \ref{thm:gen_hazard_bound_light}}
\begin{proof}
If the distribution is light-tailed, we have 
\begin{equation}\label{eq:light_tail}
\lim_{z \rightarrow \infty} \overline{F}(z)e^{\lambda^* z} < \infty \quad \text{ for some } \lambda^* > 0.
\end{equation}
This immediately implies that 
\begin{equation}\label{eq:kill_poly}
\lim_{z \rightarrow \infty} \overline{F}(z)^a z^b = 0 \quad \forall a,b > 0.
\end{equation}

Consider $\lim_{z \rightarrow \infty} \frac{f(z)}{\overline{F}(z)} = \lim_{z \rightarrow \infty} R'(z)$. If $\lim_{z \rightarrow \infty} R'(z) < \infty$ we can immediately conclude that $\sup_z \frac{f(z)}{1-F(z)} <\infty$. If $\lim_{z \rightarrow \infty} R'(z) = \infty$ instead, note that
$$
\lim_{z \rightarrow \infty} \int_{-z}^z R'(t) e^{- \delta R(t)}dt = -\frac{1}{\delta} e^{-\delta R(z) } \vert^{z=+\infty}_{z=-\infty}  = \frac{1}{\delta} < \infty.
$$
Moreover, since $\lim_{z \rightarrow \infty} R'(z) = \infty$, $R'(z) e^{- \delta R(z)}$ is strictly positive for all $z > z_0$ for some $z_0$. Furthermore, $R'(z) e^{- \delta R(z)} = \frac{f(z)}{(\overline{F}(z))^{1-\delta}}$ is eventually monotone by assumption (e),

Therefore, we can conclude that
$$
\lim_{z \rightarrow \infty}R'(z) e^{- \delta R(z)} = \frac{f(z)}{(\overline{F}(z))^{1-\delta}} = 0.
$$
$\forall \alpha \in (\delta,1)$, from Equation \eqref{eq:kill_poly} we have $\lim_{z \rightarrow +\infty} z^{\alpha}\overline{F}(z)^{\alpha-\delta} = 0$, so
$$
\lim_{z \rightarrow +\infty} \frac{f(z)z^{\alpha}}{(\overline{F}(z))^{1-\alpha}} = \lim_{z \rightarrow +\infty} \frac{f(z)}{\overline{F}(z)^{1-\delta}} \times  z^{\alpha}\overline{F}(z)^{\alpha-\delta} = 0.
$$
and hence
$$
\sup_z \frac{f(z)z^{\alpha}}{(1-F(z))^{1-\alpha}} < \infty \quad \forall \alpha \in (\delta,1).
$$
\end{proof}

\subsection{Proof of Corollary \ref{cor:lighttailed}}
\begin{proof}
For a light-tailed distribution $\cD$, we have 
$$
\lim_{z \rightarrow \infty} \overline{F}_\cD(z)e^{\lambda^* z} < \infty \quad \text{ for some } \lambda^* > 0.
$$
This implies that 
$$
\overline{F}_\cD(z) \le Ce^{-\lambda^* z} \text{ for some } C > 0, z > z_0.
$$
Let random variable $Z$ follows distribution $\cD$. Since $Z$ might take negative values, we define a new distribution $\cD'$ that only takes non-negative value by 
$$
f_{D'}(z) = \begin{cases}
\frac{1}{p_{\cD+}}f_D(z) \quad &\text{ if } z\ge 0\\
0 &\text{otherwise}
\end{cases}
$$
where $p_{\cD+} = \PP(Z \ge 0) > 0$ by right unbounded support assumption.
Clearly, with this definition of $\cD'$ we see that $\EE_{Z_1,\dots,Z_N \sim \cD}[\max\limits_i Z_i] \le \EE_{Z_1,\dots,Z_N \sim \cD'}[\max\limits_i Z_i]$ and for $z>z_0$, we have $\overline{F}_{\cD'}(z) = \frac{\overline{F}_{\cD}(z)}{p_{\cD+}} \le C'e^{-\lambda^* z}$ where $C' = \frac{C}{p_{\cD+}}$. 
Note that 
\begin{align*}
\EE_{Z_1,\dots,Z_N \sim \cD}[\max\limits_i Z_i] &\le \EE_{Z_1,\dots,Z_N \sim \cD'}[\max\limits_i Z_i] \\
&= \int_{0}^\infty \PP(\max\limits_i Z_i > x)dx \\
&\le u + \int_{u}^\infty \PP(\max\limits_i Z_i > z)dz \\
&\le u + N\int_{u}^\infty \PP(Z_i > z)dz \\
&\le u + N\int_{u}^\infty C'e^{-\lambda^* z}dz \text{ \quad assuming $u > z_0$}\\
&= u + \frac{C'N}{\lambda^*} e^{-\lambda^* u}.
\end{align*}
If we let $u = \frac{\log(N)}{\lambda^*}$, obviously $u > z_0$ if $N$ is sufficiently large. Thus, we see that 
\begin{equation}\label{eq:light_tail_exp}
\EE_{Z_1,\dots,Z_N \sim \cD}[\max\limits_i Z_i] \le \frac{\log(N)}{\lambda^*} + C' = O(N^{\xi}) \quad \forall \xi > 0.
\end{equation}
From Theorem \ref{thm:gen_hazard_bound_light} we see that $\forall \alpha \in (\delta,1)$,
\begin{equation}\label{eq:light_tail_tune}
\frac{f(z)z^{\alpha}}{(1-F(z))^{1-\alpha}} \le C_\alpha \quad \forall z \in \RR.
\end{equation}
Plug \ref{eq:light_tail_exp} and \ref{eq:light_tail_tune} into Theorem \ref{thm:master} gives the desired result.

\end{proof}

\subsection{Proof of Corollary~\ref{thm:exp_power}}

\begin{proof}
By Corollary \ref{cor:lighttailed} we only need to check that assumptions (a)-(d) hold for distribution $\cD_\beta$, exponential power family is light-tailed, and assumption (e) also holds for any $\delta \in (0,1)$. By observing the density function $f_\beta$ we can trivially see that assumptions (a)-(c) hold and that the exponential power family is light-tailed. Therefore, define 
$$
g_{\delta,\beta}(z) = \frac{f_\beta(z)}{(\overline{F}_\beta(z))^{1-\delta}} = \frac{f_\beta(z)}{(1-F_\beta(z))^{1-\delta}},
$$
it suffices to show that $\forall \delta \in [0,1), g_{\delta,\beta}(z)$ is eventually monotone.
Note that 
\begin{align*}
g_{\delta,\beta}'(z) &= \frac{f'_\beta(z)(1-F_\beta(z))^{1-\delta} + (1-\delta)(1-F_\beta(z))^{-\delta}f^2_\beta(z)}{(1-F_\beta(z))^{2-2\delta}} \\
&= \frac{C_\beta^2e^{-z^\beta}}{(1-F_\beta(z))^{2-\delta}} \times \Big(
(1-\delta)e^{-z^\beta} - \beta z^{\beta-1} \int_{z}^\infty e^{-t^\beta}dt \Big).
\end{align*}
It further suffices to show that 
$$
m_{\delta,\beta}(z) = (1-\delta)e^{-z^\beta} - \beta z^{\beta-1} \int_{z}^\infty e^{-t^\beta}dt
$$ 
is eventually non-negative or non-positive $\forall \beta > 1, \delta \in [0,1)$. Note that since $\beta > 1$,
\begin{equation}\label{eq:exppower_limsup}
\beta z^{\beta-1} \int_{z}^\infty e^{-t^\beta}dt =  \int_{z}^\infty \beta z^{\beta-1} e^{-t^\beta}dt < \int_{z}^\infty \beta t^{\beta-1} e^{-t^\beta}dt = e^{-z^\beta}.
\end{equation}
Therefore, $m_{0,\beta}(z) > 0$ for all $z \ge 0$, i.e, the hazard rate is always increasing and assumption (d) is satisfied. Now, we are left to show that $m_{\delta,\beta}(z)$ is eventually non-negative or non-positive for any $\delta \in (0,1)$. Note that
\begin{align*}
\beta z^{\beta-1} \int_{z}^\infty e^{-t^\beta}dt &= \beta (\frac{z}{z+1})^{\beta-1} (z+1)^{\beta-1} \int_{z}^\infty e^{-t^\beta}dt \\
&\ge \beta (\frac{z}{z+1})^{\beta-1} (z+1)^{\beta-1} \int_{z}^{z+1} e^{-t^\beta}dt \\
&\ge (\frac{z}{z+1})^{\beta-1}  \int_{z}^{z+1} \beta t^{\beta-1} e^{-t^\beta}dt \\
&= (\frac{z}{z+1})^{\beta-1} \Big(e^{-z^\beta} - e^{-(z+1)^\beta}\Big).
\end{align*}
Therefore,
\begin{align*}
\liminf_{z \to \infty} \frac{\beta z^{\beta-1} \int_{z}^\infty e^{-t^\beta}dt}{e^{-z^\beta}} &\ge  \liminf_{z \to \infty} \frac{(\frac{z}{z+1})^{\beta-1} \Big(e^{-z^\beta} - e^{-(z+1)^\beta}\Big)}{e^{-z^\beta}} \\
&= \lim_{z \to \infty} (\frac{z}{z+1})^{\beta-1} - \lim_{z \to \infty}(\frac{z}{z+1})^{\beta-1} e^{z^\beta - (z+1)^{\beta}} \\
&= 1.
\end{align*}
From Equation \eqref{eq:exppower_limsup} we know that 
$$
\limsup_{z \to \infty} \frac{\beta z^{\beta-1} \int_{z}^\infty e^{-t^\beta}dt}{e^{-z^\beta}} \le 1.
$$
Hence, we conclude that 
$$
\lim_{z \to \infty} \frac{\beta z^{\beta-1} \int_{z}^\infty e^{-t^\beta}dt}{e^{-z^\beta}} = 1,
$$
which implies that $m_{\delta,\beta}(z)$ is eventually non-positive for any $\delta \in (0,1)$, i.e, assumption (e) holds for any $\delta \in (0,1)$.

\end{proof}

\end{document}